%% file: skalse_235.tex
\documentclass[accepted]{uai2023} 

\usepackage[american]{babel}

\usepackage{natbib} 
\usepackage{mathtools} 
\usepackage{booktabs} 
\usepackage{tikz} 

\title{On the Limitations of Markovian Rewards\\ to Express Multi-Objective, Risk-Sensitive, and Modal Tasks}

\author[1,2]{\href{mailto:<joar.skalse@cs.ox.ac.uk>?Subject=Your UAI 2023 paper}{Joar~Skalse}{}}
\author[1]{Alessandro~Abate}
\affil[1]{%
    Computer Science Department\\
    Oxford University\\
    Oxford, UK
}
\affil[2]{%
    The Future of Humanity Institute\\
    Oxford, UK\\
}

\usepackage{amsthm, amsfonts, bbm, csquotes, amssymb} 

\newtheorem{proposition}{Proposition}
\newtheorem{theorem}{Theorem}
\newtheorem{lemma}{Lemma}
\newtheorem{corollary}{Corollary}

\newtheorem{definition}{Definition}

\usepackage{xcolor}

\newcommand{\M}{\mathcal{M}}

\newcommand{\States}{\mathcal{S}}
\newcommand{\Actions}{\mathcal{A}}
\newcommand{\mcS}{\mathcal{S}}
\newcommand{\mcA}{\mathcal{A}}
\newcommand{\init}{\mu_0}
\newcommand{\R}{R}
\newcommand{\y}{\gamma}
\newcommand{\Rs}{\textbf{R}}

\newcommand{\Ob}{{\mathcal{O}}}

\newcommand{\m}{m_{\tau,\init,\gamma}}

\newcommand{\SxA}{\mcS \times \mcA}
\newcommand{\SxAxS}{\mcS \times \mcA \times \mcS}

\newcommand{\MDP}{\langle \mcS, \mcA, \tau, \init, \R, \y \rangle}
\newcommand{\MOMDP}{\langle \mcS, \mcA, \tau, \init, \Rs, \y \rangle}
\newcommand{\MDPwO}{\langle \mcS, \mcA, \tau, \init, \tilde{\R}, \y \rangle}

\begin{document}

\maketitle

\begin{abstract}
In this paper, we study the expressivity of scalar, Markovian reward functions in Reinforcement Learning (RL), and identify several limitations to what they can express.
Specifically, we look at three classes of RL tasks; multi-objective RL, risk-sensitive RL, and modal RL.
For each class, we derive necessary and sufficient conditions that describe when a problem in this class can be expressed using a scalar, Markovian reward. Moreover, we find that scalar, Markovian rewards are unable to express most of the instances in each of these three classes. We thereby contribute to a more complete understanding of what standard reward functions can and cannot express.
In addition to this, we also call attention to modal problems as a new class of problems, since they have so far not been given any systematic treatment in the RL literature. We also briefly outline some approaches for solving some of the problems we discuss, by means of bespoke RL algorithms.
\end{abstract}

\section{Introduction}





To solve a task using reinforcement learning (RL), we must first encode that tasks as a reward function \citep{sutton2018reinforcement}.
Typically, these rewards are \emph{scalar} and \emph{Markovian}. 
However, it is often not straightforward to determine if a given task \emph{can} be adequately expressed using such a reward function. Therefore, understanding the expressivity of scalar, Markovian rewards is a basic and foundational question of the RL setting.
In this paper, we identify and characterise several specific limitations in the expressivity of scalar, Markovian rewards.
Specifically, we examine three broad classes of tasks, all of which are both intuitive to understand, and useful in many practical situations. 
We then derive necessary and sufficient conditions that describe when these tasks can be expressed using ordinary reward functions, and consequently
show that \emph{almost no} tasks in any of these three classes can be expressed using scalar, Markovian rewards.
This suggests that scalar, Markovian reward functions are semantically limited in certain important ways. 
We thus contribute to a more complete understanding of what standard reward functions can and cannot express. This clarifies the implicit assumptions behind many common RL techniques, and makes it easier to determine if they are applicable to a given practical problem.

The first class of problems we look at, in Section~\ref{section:morl}, are single-policy, multi-objective RL tasks (MORL).
In such problems, the agent receives multiple reward signals, and the aim is to learn a single policy that achieves an optimal trade-off amongst those rewards, according to some specified criterion \citep{Roijers2013,Liu2015}. 
For example, a single-policy MORL algorithm might attempt to maximise the rewards lexicographically \citep{lmorl}.
We will provide necessary and sufficient conditions describing when a MORL problem can be reduced to scalar-reward RL, by providing a single reward function that induces the same preferences as the MORL problem.
We find that this can \emph{only} be done for MORL problems that correspond to a linear weighting of the rewards, which means that it cannot be done for the vast majority of all interesting MORL problems. 
This result is analogous to Harsanyi's Utilitarian Theorem \cite{harsanyi1955cardinal}, generalised to the RL setting.

The next class of problems we study, in Section~\ref{section:risk_sensitive_rl}, is risks-sensitive RL. 
In expected utility theory, risk-aversion is often modelled using utility functions that are concave in some of their variables. 
We will show that these tasks cannot be expressed as Markovian reward functions,
by demonstrating that no non-affine monotonic transformations of the trajectory return function are possible. 
This demonstrates another limitation in the expressive power of Markovian rewards. 

In Section~\ref{section:modal}, we introduce a new class of tasks, which we call \emph{modal} tasks. These are tasks where the agent is evaluated not only based on what distribution of trajectories it generates, but also based on what it \emph{could have done} along those trajectories. 
As an example, consider the instruction \enquote{you should always be \emph{able} to return to the start state}. 
We provide a formalisation of such tasks, argue that there are many situations in which these tasks could be useful, and finally prove that these tasks also typically cannot be formalised using scalar, Markovian reward functions. 

In Section~\ref{section:solving_inexpressible}, we discuss how to solve tasks from each of these classes using specialised RL solutions: we provide references to existing literature, and also sketch both an approach for learning a wide class of MORL problems, and an approach for learning a wide class of modal problems. Finally, in Section~\ref{section:discussion}, we discuss the implications of our results, together with several pieces of related work.

\section{Preliminaries}\label{section:preliminaries}

The standard RL setting is formalised using \textit{Markov Decision Processes} (MDPs) \citet{sutton2018reinforcement}, which are tuples $\MDP$ where $\mcS$ is a set of states, $\mcA$ is a set of actions, $\tau : \mcS \times \mcA \to \Delta(\States)$ is a transition function, $\init$ is an initial state distribution over $\mcS$, $R : \mcS \times \mcA \to \mathbb{R}$ is a reward function, and $\gamma \in (0,1)$ is a discount factor.
A \emph{trajectory} $\xi$ is in general an element of $(\SxA)^\omega$, i.e.\ a sequence $s_0, a_0, s_1 \dots$. 
We use $G$ to denote the \emph{trajectory return function}, where $G(\xi) = \sum_{t=0}^\infty \gamma^t R(s_t,a_t)$.
A \emph{policy} is a mapping $\pi : \mcS \to \Delta(\mcA)$, and $\Pi$ is the set of all policies. 
Given a policy $\pi$, its \emph{value function} $V^\pi : \mcS \to \mathbb{R}$ is the function where $V^\pi(s)$ is the expected future discounted reward when following $\pi$ from $s$, and its \emph{$Q$-function} $Q^\pi$ is $\mathbb{E}_{s' \sim \tau(s,a)}[R(s,a) + \gamma \cdot V^\pi(s')]$. 
The \emph{policy evaluation function} $J : \Pi \to \mathbb{R}$ is $J(\pi) = \mathbb{E}_{s_0 \sim \init}[V^\pi(s_o)]$. If a policy maximises $J$, then we say that this policy is \emph{optimal}. We denote optimal policies by $\pi^\star$, and their value function and $Q$-function by $V^\star$ and $Q^\star$. Moreover, given an MDP $\M$, we say that $\M$'s policy order is the ordering $\prec$ on $\Pi$ where $\pi_1 \prec \pi_2 \iff J(\pi_1) < J(\pi_2)$ for any $\pi_1, \pi_2$. 

In this paper, we will say that a reward function $R$ is \emph{trivial} if $J(\pi_1) = J(\pi_2)$ for all $\pi_1,\pi_2$. Moreover, we say that $R_1$ and $R_2$ are \emph{equivalent} if $J_1(\pi_1) < J_1(\pi_2) \iff J_2(\pi_1) < J_2(\pi_2)$ for all $\pi_1,\pi_2$, and that they are \emph{opposites} if $J_1(\pi_1) < J_1(\pi_2) \iff J_2(\pi_1) > J_2(\pi_2)$ for all $\pi_1,\pi_2$.

MORL problems are formalised using \emph{Multi-Objective MDPs} (MOMDPs), which are tuples $\MOMDP$, 
with the only difference from MDPs being $\Rs$, 
which is now a function $\Rs : \mcS \times \mcA \to \mathbb{R}^k$ that, for each pair $(s,a)$, 
returns $k$ different rewards (for some finite $k$). We denote the $i$'th component of $\Rs$ as the scalar reward function $R_i$, 
and use $V^\pi_i$, $Q^\pi_i$, $J_i$, and $G_i$, etc, to refer to its value-, $Q$-, evaluation-, and return function, etc. 
There are two types of MORL problems; single-policy MORL, where the goal is to compute one policy that achieves an optimal trade-off of the rewards, and multi-policy MORL, where the aim is to compute several policies (typically with the aim of approximating the Pareto front of the rewards).
In this paper, we are concerned with single-policy MORL.
Since there may not be a single policy that maximises each component of $\Rs$, a single-policy MORL problem needs some additional rule for combining and trading off each reward. 

In economics and psychology, risk-aversion is often modelled using utility functions $U(c)$ that are concave in some relevant variable $c$.
The most common risk-averse utility functions are the \emph{exponential}, the \emph{isoelastic}, and the \emph{quadratic} utility functions. 
The exponential utility function is given by $U(c) = 1-e^{\alpha c}$, where $\alpha > 0$ is a parameter controlling the degree of risk aversion. The isoelastic utility function is given by $U(c) = (c^{1-\alpha}-1)/(1-\alpha)$, for $\alpha > 0, \alpha \neq 1$, or by $U(c) = \ln(c)$ (corresponding to the case when $\alpha = 1$). The quadratic utility function is given by $U(c) = c - \alpha c^2$, where $\alpha > 0$.
Since this function is decreasing for sufficiently large $c$, its domain is typically restricted to $(-\infty, 1/2\alpha]$.

\paragraph*{A Remark on \enquote{Tasks}:}

In this paper, we are investigating the question of when a given task can be expressed using a scalar, Markovian reward function.
To do this, we must first formalise what it should mean for a reward function to \enquote{express a task}. 
One option is to say that a task corresponds to a desired policy $\pi$, and that a reward function $R$ expresses the task if $\pi$ is optimal under $R$ (possibly with the additional requirement that $\pi$ is the \emph{only} policy that is optimal under $R$). With this definition, we find that \emph{any} task can be expressed as a Markovian reward function, at least as long as $\pi$ is stationary and deterministic (see Appendix~B).
With this definition, the problem is therefore rather trivial.

An alternative, stronger formalisation is to say that a task corresponds to an ordering $\prec$ on $\Pi$, which encodes a preference ordering over all policies, and that a reward function $R$ expresses the task if its corresponding evaluation function $J$ orders $\Pi$ according to $\prec$. It is primarily this latter definition that we will use in this paper. 
The main reason for this is that it is often impossible to find the optimal policy in complex environments. 
For example, in a robotics problem, it is typically not feasible to find a policy that is globally optimal.
This means that it is not enough for $R$ to admit the correct optimal policy; it must also induce the right preferences between the all the (sub-optimal) policies that the policy synthesis algorithm might in fact generate. The only way to robustly ensure that this is the case is if $R$ induces the right policy ordering.
For this reason, we think it is more informative to think of a problem setting (i.e.\ a \enquote{task}) as corresponding to an ordering on $\Pi$.


\section{Multi-Objective Problems}\label{section:morl}

In this section, we examine the MORL setting.
We first need a general definition of what a single-policy MORL problem is. Recall that a MOMDP $\MOMDP$ by itself has no one canonical objective to maximise. We therefore introduce the notion of a \emph{MORL objective}:
\begin{definition}\label{def:morl_objective}
A \textbf{MORL objective} over $k$ rewards is a function $\Ob$ that takes $k$ policy evaluation functions $J_1 \dots J_k$ and returns a (total) ordering $\prec_\Ob$ over the set of all policies $\Pi$.
\end{definition}
Given a MOMDP $\M$, a MORL objective $\Ob$ gives us an ordering $\prec_\Ob$ over $\Pi$ that tells us when a policy is preferred over another. 
For the purposes of this paper, we will not need to impose any further requirements on $\prec_\Ob$. For example, we will not insist that $\prec_\Ob$ must have a greatest element in $\Pi$, or that $\pi_1 \prec_O \pi_2$ whenever $\pi_2$ is a Pareto improvement over $\pi_1$, etc, even though a reasonable MORL objective presumably would have these properties.
We next provide a few examples of MORL objectives, where we denote by $\pi_1, \pi_2$ any given pair of distinct policies.   

\begin{definition}\label{def:lexmax}
Given $J_1 \dots J_k$, the \textbf{LexMax} objective $\prec_\texttt{Lex}$ is given by $\pi_1 \prec_\texttt{Lex} \pi_2$ iff there is an $i \in \{1 \dots k\}$ such that $J_i(\pi_1) < J_i(\pi_2)$ and $J_j(\pi_1) = J_j(\pi_2)$ for all $j < i$.
\end{definition}

\begin{definition}\label{def:maxmin}
Given $J_1 \dots J_k$, the \textbf{MaxMin} objective $\prec_\texttt{Min}$ is given by $\pi_1 \prec_\texttt{Min} \pi_2 \iff \min_i J_i(\pi_1) < \min_i J_i(\pi_2)$.
\end{definition}

\begin{definition}\label{def:maxsat}
Given $J_1 \dots J_k$ and some $c_1 \dots c_m \in \mathbb{R}$, the \textbf{MaxSat} objective $\prec_\texttt{Sat}$ is given by $\pi_1 \prec_\texttt{Sat} \pi_2$ if and only if
the number of rewards that satisfy $J_i(\pi_1) \geq c_i$ is larger than the number of rewards that satisfy $J_i(\pi_2) \geq c_i$.
\end{definition}


\begin{definition}\label{def:consat}
Given $J_1, J_2$ and some $c\in \mathbb{R}$, the \textbf{ConSat} objective $\prec_\texttt{Con}$ is given by $\pi_1 \prec_\texttt{Con} \pi_2$ if and only if
either $J_1(\pi_1) < c$ and $J_1(\pi_1) < J_1(\pi_2)$, or $J_1(\pi_1), J_1(\pi_2) \geq c$ and $J_2(\pi_1) < J_2(\pi_2)$.
\end{definition}

In other words, the LexMax objective has \emph{lexicographic} preferences over $R_1 \dots R_m$, so that policies are first ordered by their expected discounted $R_1$-reward, and then policies that obtain the same expected discounted $R_1$-reward are ordered by their expected discounted $R_2$-reward, and so on. The MaxMin objective orders policies by their \emph{worst} performance according to any of $R_1 \dots R_m$ (which could be used to obtain worst-case guarantees). The MaxSat objective only cares whether a policy reaches a certain \emph{threshold} for each reward, and ranks policies based on how many thresholds they reach. The ConSat objective aims to maximise $J_2$, but under the constraint that $J_1$ reaches a certain threshold. 
Note that these objectives are not necessarily the most important MORL objectives.
Rather, they are simply a short list of illustrative examples, meant to demonstrate the flexibility of the MORL framework, and give an intuition for what types of problems it can be used to express.
A few more examples can be found in Appendix~C.

We next define what it means to \emph{reduce} a MORL problem to a scalar RL problem.
Given a MORL objective $\Ob$ and a MOMDP $\M$, we use $\prec_\Ob^\M$ to denote the ordering we get when we apply $\Ob$ to $\M$'s policy evaluation functions:


%
\begin{definition}
A MOMDP $\M = \MOMDP$ with MORL objective $\Ob$ is \textbf{equivalent} to the MDP $\M' = \MDP$ if and only if $\M'$'s policy order is $\prec_\Ob^\M$. 
We then say that $\M$ with $\Ob$ is \textbf{scalarized} by $R$.
If $\M$ with $\Ob$ is scalarized by some $R$ then we say that $\M$ with $\Ob$ is \textbf{scalarizable}, otherwise we say that it is \textbf{unscalarizable}. 
\end{definition}
Note that $\M'$ must have the same states, actions, transition function, initial state distribution, and discount factor, as $\M$. This definition therefore says that $\M$ with $\Ob$ is equivalent to $\M'$ if $\M'$ is given by replacing $\Rs = \langle R_1 \dots R_k \rangle$ with a single reward function $R$, and $R$ induces the same preferences between all policies as $\Ob(J_1 \dots J_k)$.
Note also that we require $R$ to express the same \emph{policy order} as $\Ob(J_1 \dots J_k)$; it is not enough for $R$ and $\Ob(J_1 \dots J_k)$ to have the same \emph{optimal policies} (see Section~\ref{section:preliminaries}).

Given this definition, we can now provide the necessary and sufficient conditions for when a MORL problem can be reduced to a scalar-reward RL problem. All proofs are provided in the supplementary material.


\begin{theorem}\label{thm:linearity_thm}
If a MOMDP $\M$ with objective $\Ob$ is scalarizable, then there exist $w_1 \dots w_k \in \mathbb{R}$ such that $\M$ with $\Ob$ is scalarized by the reward $R(s,a) = \sum_{i=1}^k w_i \cdot R_i(s,a)$.
\end{theorem} 

Theorem~\ref{thm:linearity_thm} tells us that a MORL objective can be expressed using a scalar, Markovian reward function if and only if that objective corresponds to a linear weighting of the individual rewards.
In other words, scalar, Markovian rewards are unable to express all non-linear MORL problems.
As we will see, this imposes a strong limitation on what MORL tasks can be encoded using scalar, Markovian rewards.

It is worth noting that Theorem~\ref{thm:linearity_thm} is analogous to Harsanyi's Utilitarian Theorem \cite{harsanyi1955cardinal} from social choice theory, but generalised to the RL setting. In brief, this theorem supposes that we have a finite set of outcomes $\Omega$ and a group of individuals $\{1 \dots k\}$ with different preferences over $\Omega$, and that we wish to construct an aggregate preference structure that captures the preferences of the group. 
Moreover, also suppose that (1) the preferences of each individual $i$ are described by a utility function $U_i : \Omega \to \mathbb{R}$, (2) the aggregate preferences of the group are described by a further utility function $U_G : \Omega \to \mathbb{R}$, and (3) for all distributions $\mathcal{D}_1$ and $\mathcal{D}_2$ over $\Omega$, if $\mathbb{E}_{O \sim \mathcal{D}_1}[U_i(O)] = \mathbb{E}_{O \sim \mathcal{D}_2}[U_i(O)]$ for every individual $i$, then $\mathbb{E}_{O \sim \mathcal{D}_1}[U_G(O)] = \mathbb{E}_{O \sim \mathcal{D}_2}[U_G(O)]$. Harsanyi's Utilitarian Theorem then says that $U_G$ must be given by some linear combination of $U_1 \dots U_k$. 
The link to Theorem~\ref{thm:linearity_thm} becomes clear if we think of $\Omega$ as being the set of all trajectories which are possible in a MOMDP $\M$, $U_1 \dots U_k$ as being the trajectory return functions $G_1 \dots G_k$ of the reward functions $R_1 \dots R_k$ in $\M$, and $U_G$ as being the trajectory return function of the scalarizing reward $R$. However, note that Harsanyi's Utilitarian Theorem assumes that $\Omega$ is finite, whereas the set of all trajectories may be uncountably infinite. 
Moreover, assumption (3) quantifies over all possible distributions over $\Omega$, whereas Theorem~\ref{thm:linearity_thm} only quantifies over distributions that can be realised as policies in a given MOMDP $\M$. If $\Omega$ is allowed to be infinite, and assumption (3) is restricted to range over only some distributions over $\Omega$, then Harsanyi's Utilitarian Theorem does not hold in general. The generalisation provided by Theorem~\ref{thm:linearity_thm} is therefore non-trivial.

%
Theorem~\ref{thm:linearity_thm} also entails the following corollary, which is useful to elucidate when a MORL objective cannot be expressed using scalar reward functions. 
Given an ordering $\prec$ over $\Pi$, depending on some evaluation functions $J_1 \dots J_k$, we say that a function $U : \Pi \to \mathbb{R}$ \emph{represents} $\prec$ if $U(\pi_1) < U(\pi_2) \iff \pi_1 \prec \pi_2$. 
We say that $U$ is a \emph{linear representation} if $U(\pi) = f(\sum_{i=1}^k w_i \cdot J_i(\pi))$ for some $w_1 \dots w_k \in \mathbb{R}$ and some strictly monotonic $f$.


\begin{corollary}\label{cor:nonlinear_rep}
If $\Ob(J_1 \dots J_k)$ has a non-linear representation $U$, and $\M$ is a MOMDP whose $J$-functions are $J_1 \dots J_k$, then $\M$ with $\Ob$ is unscalarizable.
\end{corollary}

Therefore, we can prove that $\M$ with $\Ob$ is unscalarizable by finding a non-linear representation of $\prec_\Ob^\M$. 
Accordingly, we now show that none of the MORL objectives given in Definitions~\ref{def:lexmax}-\ref{def:consat} can be expressed using scalar, Markovian reward functions, except in a few degenerate cases. 

\begin{corollary}\label{corollary:no_lexmax}
$\M$ with $\textbf{LexMax}$ is unscalarizable, as long as $\M$ has at least two reward functions that are neither trivial, equivalent, or opposite. 
\end{corollary}

Note that if all reward functions are either trivial, equivalent, or opposite, then the only reward function that matters for $\textbf{LexMax}$ is the highest-priority non-trivial reward function. In that case, $\M$ with $\textbf{LexMax}$ is equivalent to the MDP which contains only this reward function.

\begin{corollary}\label{corollary:no_minmax}
$\M$ with $\textbf{MaxMin}$ is unscalarizable, unless $\M$ has a reward function $R_i$ such that $J_i(\pi) \leq J_j(\pi)$ for all $j \in \{1 \dots k\}$ and all $\pi$.
\end{corollary}

Note that if $\M$ has a reward function $R_i$ such that $J_i(\pi) \leq J_j(\pi)$ for all $j$ and $\pi$, then this is the only reward function that matters for the $\textbf{MaxMin}$ objective. 
In that case, $\M$ with $\textbf{MaxMin}$ is equivalent to the MDP which contains only $R_i$.

\begin{corollary}\label{corollary:no_maxsat}
$\M$ with $\textbf{MaxSat}$ is unscalarizable, as long as $\M$ has at least one reward $R_i$ where $J_i(\pi_1) < c_i$ and $J_i(\pi_2) \geq c_i$ for some $\pi_1, \pi_2 \in \Pi$.
\end{corollary}

Note that if $\M$ has no reward $R_i$ where $J_i(\pi_1) < c_i$ and $J_i(\pi_2) \geq c_i$ for some $\pi_1, \pi_2 \in \Pi$, then either all policies satisfy all constraints, or no policy satisfies any constraint. In either case, $\M$ with $\textbf{MaxSat}$ would be equivalent to an MDP with a trivial reward function.

\begin{corollary}\label{corollary:no_consat}
$\M$ with $\textbf{ConSat}$ is unscalarizable, unless either $R_1$ and $R_2$ are equivalent, or $\max_{\pi}J_1(\pi) \leq c$, or $\min_{\pi}J_1(\pi) \geq c$.
\end{corollary}

Note that if $\max_{\pi}J_1(\pi) \leq c$ then no policy satisfies the constraint, in which case $\M$ with $\textbf{ConSat}$ is equivalent to the MDP with $R_1$. If $\min_{\pi}J_1(\pi) \geq c$ then all policies satisfy the constraint, in which case $\M$ with $\textbf{ConSat}$ is equivalent to the MDP with $R_2$. If $R_1$ and $R_2$ are equivalent, then $\M$ with $\textbf{ConSat}$ is scalarized by $R_1$ or $R_2$.

Corollaries~\ref{corollary:no_lexmax}-\ref{corollary:no_consat} thus show that none of the MORL objectives given in Definition~\ref{def:lexmax}-\ref{def:consat} can be expressed using a scalar, Markovian reward function, except in a few degenerate cases where those MORL objectives are trivialised. This demonstrates that MORL problems typically cannot be scalarized in a satisfactory way.

To get an intuition for this result, note that the expected cumulative return of a Markovian reward function always is maximised by some stationary policy, whereas some of these MORL objectives may require the optimal policy to be non-stationary. For example, consider the \textbf{MaxMin} objective, and suppose the agent can choose between an action giving one $R_1$-reward, and an action giving one $R_2$-reward.
Then the optimal choice may depend on how much $R_1$ and $R_2$-reward the agent has got in the past. This means that the optimal policy may be non-stationary, and thus not correspond to any Markovian reward. 




\section{Risk-Sensitive Problems}\label{section:risk_sensitive_rl}

The next area we will look at is that of \emph{risk-sensitive} RL. An ordinary RL agent tries to maximise the \emph{expectation} of its reward function.
However, there are many cases where it is natural to require the agent to be \emph{risk-averse}.
For example, we might prefer a policy that reliably achieves $5$ reward, over one that achieves $11$ reward with probability $0.5$, and otherwise gets $0$ reward, even though the latter policy achieves a higher expected reward.
In this section, we will examine when scalar, Markovian reward functions can be used to encourage such behaviour.

In expected utility theory, risk-aversion is often modelled using concave utility functions. In particular, suppose we have a set of \emph{outcomes} $C$, each of which is associated with some utility via a function $U_1 : C \to \mathbb{R}$. We can then construct a second utility function $U_2 : C \to \mathbb{R}$ by letting $U_2(c) = f(U_1(c))$ for some concave function $f$. Then, an agent which maximises expected utility according to $U_2$, will be risk-averse with respect to utility as defined by $U_1$. 
For example, suppose each outcome is associated with some monetary payoff, and that $U_1$ measures how much money is obtained in each outcome. 
If we were to maximise expected utility according to $U_1$, then we would prefer a $50\%$ chance of obtaining $\$2,000,000$, to a certain chance of obtaining $\$900,000$. However, in the real world, most people would prefer the latter option. One reason for this is that, while getting $\$2,000,000$ is better than getting $\$900,000$, it is less than twice as good. We can model these preferences by using a second utility function $U_2$ that is concave in $U_1$.
Intuitively, $U_2$ should measure how much \emph{benefit} we get from the money.
Then the expected $U_2$-utility might be higher for the safe option than the risky option, even though the expected $U_1$-utility is higher for the risky option.

In reinforcement learning, the \emph{outcomes} are the trajectories that might occur in the environment, and the \emph{utility} of a trajectory $\xi$ is induced by the reward function $R$ via the return function, $G$. 
If the transition function $\tau$ is nondeterministic, then the agent cannot reliably enact a particular outcome (i.e., move along a particular trajectory), but can instead only choose between some distributions over outcomes. 
By default, the agent may then be compelled to pursue a policy that achieves a high reward with small probability, as long as the expectation remains high.
A natural question is then whether we could avoid this by constructing a second reward function that is concave in the original reward function, similar to what is done in expected utility theory. That is, given a reward function $R_1$ and a concave function $f$, can we construct a second reward function $R_2$ such that $G_2 = f(G_1)$? Our next theorem demonstrates that this is impossible. As before, the proof is in the appendix.

\begin{theorem}\label{thm:risk_theorem}
Given $\States$, $\Actions$, and $\gamma$, 
let $R_1$ and $R_2$ be two reward functions.
If $\gamma \geq 0.5$, and for all $\xi_1,\xi_2 \in (\SxA)^\omega$, 
$$
G_1(\xi_1) \leq G_1(\xi_2) \iff G_2(\xi_1) \leq G_2(\xi_2),
$$
then $\exists a \in \mathbb{R}$, $b \in \mathbb{R} > 0$ such that for all $\xi \in (\SxA)^\omega$,
$$
G_1(\xi) = b \cdot G_2(\xi) + a.
$$
\end{theorem}

Theorem~\ref{thm:risk_theorem} effectively tells us that \emph{only affine transformations of $G$ are possible}. 
From this result, it straightforwardly follows that none of the standard risk-averse utility functions (exponential utility, isoelastic utility, and quadratic utility) can be expressed using Markovian reward functions:

\begin{corollary}
For any non-trivial reward $R_1$ and any constant $\alpha \neq 0$, if $\gamma \geq 0.5$ then there is no reward $R_2$ such that $G_2(\xi) = -e^{\alpha G_1(\xi)}$ for all $\xi \in (\SxA)^\omega$.
\end{corollary}

\begin{corollary}
For any non-trivial reward $R_1$ and any constant $\alpha > 0$, $\alpha \neq 1$, if $\gamma \geq 0.5$ then there is no reward $R_2$ such that $G_2(\xi) = G_1(\xi)^{1-\alpha}$ for all $\xi \in (\SxA)^\omega$.
\end{corollary}

\begin{corollary}
For any non-trivial reward $R_1$, if $\gamma \geq 0.5$ then there is no reward $R_2$ such that $G_2(\xi) = \ln (G_1(\xi))$ for all $\xi \in (\SxA)^\omega$.
\end{corollary}

\begin{corollary}
For any non-trivial reward $R_1$ and any $\alpha > 0$ where $\max_\xi G_1(\xi) \leq \frac{1}{2\alpha}$, if $\gamma \geq 0.5$ then there is no reward $R_2$ such that $G_2(\xi) = G_1(\xi) - \alpha G_1(\xi)^2$ for all $\xi \in (\SxA)^\omega$.
\end{corollary}

Theorem~\ref{thm:risk_theorem} thus implies that none of the standard risk-averse utility functions can be expressed using scalar, Markovian reward functions. To get an intuition on Theorem~\ref{thm:risk_theorem}, consider the fact that the expected cumulative return of a Markovian reward function always is maximised by some stationary (i.e.\ Markovian) policy. However, a risk-averse objective may require the optimal policy to be non-stationary, because whether or not it is worth taking a particular gamble could depend on how much reward you have accrued in the past. This suggests that there should be instances where risk-sensitive objectives cannot be expressed as Markovian reward functions. Theorem~\ref{thm:risk_theorem} formalises this intuition. 

It is also worth remarking on the fact that Theorem~\ref{thm:risk_theorem} considers the value of $G_1$ and $G_2$ for \emph{all} trajectories in $(\SxA)^\omega$. For any particular transition function $\tau$, most of these trajectories are likely to be impossible (unless $\tau$ allows you to transition between any two states via any action with non-zero probability). We could therefore alternatively consider the condition where $G_1(\xi_1) \leq G_1(\xi_2) \iff G_2(\xi_1) \leq G_2(\xi_2)$ for those trajectories $\xi_1,\xi_2$ that are possible in a given environment. In this case, it \emph{can} be possible for $G_2$ to be non-affine in $G_1$. For example, consider the case of a tree-shaped MDP, where $\tau(s,a) = s$ and $R_1(s,a) = 0$ for all actions $a$ if $s$ is a leaf-node. In that case, $G_2$ can be an arbitrary transformation of $G_1$. However, to construct the corresponding reward function $R_2$, we would need to have a detailed understanding of the environment (which is against the main tenet in RL), and furthermore the resulting reward function would no longer induce the same behaviour if it were used in a different environment. For this reason, we believe that it is more relevant to consider the set of all trajectories in $(\SxA)^\omega$. Nonetheless, an interesting direction for further work could be to more extensively study what happens if the set of trajectories under consideration is restricted in various ways.

Finally, note that Theorem~\ref{thm:risk_theorem} assumes that the discount parameter $\gamma \geq 0.5$. It is not clear if this is strictly necessary, so it might be possible to generalise Theorem~\ref{thm:risk_theorem} by removing this requirement. This would, however, require a different proof strategy.
Nonetheless, this assumption is not very restrictive, as in practice $\gamma$ is almost always set to be greater than $0.5$ (typically $\gamma \geq 0.9$).

\section{Modal Problems}\label{section:modal}





The final class of problems that we will examine is a class of tasks that we refer to as \emph{modal} tasks. Before we give a formal definition of this class, we will first provide some intuition. In analytic philosophy, a distinction is made between \emph{categorical} facts and \emph{modal} facts. In short, categorical facts only concern what is true in actuality, whereas modal facts concern what must be true, could have been true, or cannot be true, etc. 
For example, it is a categorical fact that the Eiffel Tower is brown, and a modal fact that it \emph{could have had} a different colour. It is (arguably) a categorical fact that the number 3 is prime, and a modal fact that it \emph{could not have been} otherwise.
To give another example, there is a difference between stating that \emph{nothing can travel faster than light} and that \emph{nothing does travel faster than light} -- the former statement, which is modal, is stronger than the latter, which is categorical.
One can further distinguish between different kinds of possibility (e.g.\ logical vs physical possibility, etc), and discussions about modality also involves topics such as \emph{causality} and \emph{counterfactuals}, etc.
A complete treatment of this subject is beyond the scope of this paper, but for an overview see \citet{sep-possible-worlds}.


Modality does of course relate to \emph{modal logic}, and thus also to \emph{temporal logic}. In particular, computational tree logic (CTL, see e.g.\ \cite{BaierKatoen2008}), and its extensions, can express many modal statements.
\footnote{To avoid a possible confusion, we should emphasise that we here use the term \enquote{modal} in a somewhat more narrow sense than the sense of \enquote{modal logic}. In particular, we use it to mean \enquote{pertaining to what is possible or impossible}, as in e.g.\ \citet{sep-modality-varieties}. In that sense, Linear Temporal Logic (LTL) does not express modal statements, even though it is a modal logic, because LTL can only make assertions about what in fact occurs. For that reason, not everything that relates to modal logic will be related to the setting we discuss here. The type of possibility we discuss is specifically \enquote{possibility according to the transition function}.}


The intuition behind this section is that a reward function always is expressed in terms of categorical facts, whereas many tasks are naturally expressed in terms of modal facts.
For example, consider an instruction such as \enquote{you should always be \emph{able} to return to the start state}. This instruction seems quite reasonable, but it is not obvious how to translate it into a reward function. Note that this instruction is not telling the agent to \emph{actually} return to the start state, it merely says that it should maintain the \emph{ability} to do so.
This illustrates the motivation behind modal tasks; they let us reward the agent based on what is \emph{possible} or \emph{impossible} along its trajectory, rather than just in terms of what in fact occurs along that trajectory.
Given this background motivation, we can now give a formal definition of modal tasks:


\begin{definition}\label{definition:modal_reward}
Given a set of states $\mcS$ and a set of actions $\mcA$, a \textbf{modal reward function} $R^\Diamond$  is a function $R^\Diamond : \SxA \times (\mcS \times \mcA \to \Delta(\mcS)) \to \mathbb{R}$ which takes a state $s \in \mcS$, an action $a \in \mcA$, and a transition function $\tau$ over $\mcS$ and $\mcA$, and returns a real number.
\end{definition}

$R^\Diamond(s,a,\tau)$ is the reward that is obtained when taking action $a$ in state $s$ in an environment whose transition function is $\tau$.
Here we allow $R^\Diamond$ an unrestricted dependence on $\tau$, to make our results as general as possible, even if a practical algorithm for solving modal tasks presumably would require restrictions on what this dependence can look like (see Appendix~D).
Modal reward functions can be used to express instructions such as that we gave above.
For example, a simple case might be \enquote{you get 1 reward if you reach this goal state, and -1 reward if you ever enter a state from which you cannot reach the initial state}. This reward depends on the transition function, because the transition function determines from which states you can reach the initial state.
As usual, $R^\Diamond$ then induces a $Q$-function $Q^\Diamond$, value function $V^\Diamond$, and evaluation function $J^\Diamond$, etc.
We say that a modal reward $R^\Diamond$ and an ordinary reward $R$ are \emph{contingently equivalent} given a transition function $\tau$ if $J^\Diamond$ and $J$ induce the same ordering of policies given $\tau$, and that they are \emph{robustly equivalent} if $J^\Diamond$ and $J$ induce the same ordering of policies for all $\tau$. 
We use $R^\Diamond_\tau$ to denote the reward function $R^\Diamond_\tau(s,a) = R^\Diamond(s,a,\tau)$.
We will also use the following definition. 

\begin{definition}
A modal reward function $R^\Diamond$ is \textbf{vacuous} if there is a reward function $R$ such that for all $\tau$, $R$ and $R^\Diamond_\tau$ have the same policy ordering under $\tau$.
\end{definition}


The intuition here is that a vacuous modal reward function does not actually depend on $\tau$ in any important sense. Note that this is \emph{not} necessarily to say that $R^\Diamond_{\tau} = R$ for all $\tau$. For example, it could be the case that $R^\Diamond_{\tau}$ is a \emph{scaled} version of $R$, or that $R^\Diamond_{\tau}$ and $R$ differ by \emph{potential shaping} \citet{ng1999}, or that $R^\Diamond_\tau$ is modified in a way such that $\mathbb{E}_{S' \sim \tau(s,a)}[R^\Diamond_\tau(s,a,S')] = \mathbb{E}_{S' \sim \tau(s,a)}[R(s,a,S')]$,
since none of these differences affect the policy ordering (for a more in-depth examination, see \citet{invariance_ambiguity}).
From this, we get the following straightforward result:





\begin{theorem}
For any modal reward $R^\Diamond$ and any transition function $\tau$, there exists a reward $R$ that is contingently equivalent to $R^\Diamond$ given $\tau$. Moreover, unless $R^\Diamond$ is vacuous, there is no reward that is robustly equivalent to $R^\Diamond$.
\end{theorem}

In other words, \emph{every} modal task can be expressed with an ordinary reward function in each particular given environment, but \emph{no} reward function expresses a (non-vacuous) modal task in all environments. Is this enough? We argue that it is not, because the construction of $R^\Diamond_\tau$ will invariably be laborious, and require detailed knowledge of the environment. For example, consider the task \enquote{you should always be able to return to the start state}; here, constructing $R^\Diamond_\tau$ would amount to manually enumerating all the states from which the start state is reachable: this would be very much against the spirit of RL, where much of the point is that we want to be able to specify tasks which can be pursued in \emph{unknown} environments. In short, a method which requires a model of the environment is arguably not an RL method. We thus argue that reward functions are largely unable to capture modal tasks in a satisfactory way.


One remaining question might be why one would want to express tasks for RL agents in terms of modal properties. After all, what benefit is there to the instruction \enquote{never enter a state from which it is possible to quickly enter an unsafe state} over the instruction \enquote{never enter an unsafe state}? One reason is that the former task might lead to behaviour that is more robust to changes in the environment. For example, if an RL agent is trained in a simulated environment, and deployed in the real world, then it seems like it would be preferable to tell the agent to avoid \emph{risky} states, rather than \emph{unsafe} states, since imperfections in the simulation could lead to an underestimation of the risk involved.
Another example is the existing work on avoiding side effects (e.g.\ \citep{krakovnaside, krakovnaside2, turnerside, griffin2022alls}), which it is natural to express in modal terms. This work can be viewed as being aimed at making the behaviour of an RL agent more robust to misspecification of the reward function.




\section{Solving Tasks That Are Inexpressible by Markovian Rewards}\label{section:solving_inexpressible}

We have pointed to three broad classes tasks that cannot be expressed using scalar, Markovian reward functions, namely multi-objective, risk-sensitive, and modal tasks. A natural next question is whether these tasks \emph{can} be solved at all using RL, or whether only tasks corresponding to Markovian reward functions can be effectively learnt.  
We briefly discuss this issue below. In short, it is indeed possible to design RL algorithms for tasks in each of these categories. 

First of all, the existing literature already contains several bespoke RL algorithms that solve some of the problems that we have discussed.
Multi-objective reinforcement learning is particularly well-explored, with many existing algorithms. Most of these algorithms are designed to solve a specific MORL objective; for example, \cite{lmorl} solve the \textbf{LexMax} objective, and \cite{tessler2019} solve the \textbf{ConSat} objective. 
Similarly, there are existing algorithms for risk-sensitive RL (e.g.\ \cite{chow2015riskconstrained}), and even algorithms that solve certain modal tasks 
\citep{krakovnaside, krakovnaside2, turnerside, pctl_rl, griffin2022alls}. 
We give a more complete overview of this existing work in Section~\ref{section:related_work}.

It should also be possible to design algorithms that can flexibly solve many different tasks from the classes we have discussed, instead of having to be designed for just one particular task. For example, suppose a MORL objective can be represented by a function $U : \mathbb{R}^k \to \mathbb{R}$, such that $\pi_1 \prec \pi_2$ when $U(J_1(\pi_1) \dots J_k(\pi_1)) < U(J_1(\pi_2) \dots J_k(\pi_2))$, and that $U$ is \emph{differentiable}. We give a few examples of such objectives in
Appendix~C, including e.g.\ a \enquote{soft} version of \textbf{MaxMin}. 
With such an objective, if we have a policy $\pi$ that is differentiable with respect to some parameters $\theta$, then one could compute the gradient of $U(J_1(\pi)\dots J_k(\pi))$ with respect to $\theta$, and then use a policy gradient method to increase $U$. This means that it should be possible to design an actor-critic algorithm that can solve any differentiable MORL objective. We consider the development of such methods to be a promising direction for further work. 

In Appendix~D, we also outline a possible approach for solving a wide class of modal tasks. Further exploration of this setting would also be interesting for further work.

\section{Discussion}\label{section:discussion}

In this paper, we have studied the ability of Markovian reward functions to express different kinds of problems. We have looked at three classes of tasks; multi-objective tasks, risk-sensitive tasks, and modal tasks, and found that Markovian reward functions are unable to express most of the tasks in each of these three classes.
In particular, have provided necessary and sufficient conditions for when a single-policy MORL problem can be expressed using a scalar, Markovian reward function, and demonstrated that this only can be done when the MORL objective corresponds to a linear weighting of the individual rewards. 
Moreover, we have also provided necessary and sufficient conditions for when a monotonic transformation of the return function, $G$, can be expressed as a Markovian reward function, and demonstrated that this only can be done for affine transformations.
Furthermore, we have also also drawn attention to a class of tasks which have just barely been explored previously (namely modal tasks), and shown that most of these tasks cannot be expressed using Markovian reward functions.
Finally, we have shown that many of these problems still can be solved with RL, and even outlined some methods for doing this.

Our work has a number of immediate practical implications. First of all, we have contributed to a more precise demarcation of what types of problems can be expressed within the most common RL formalism. This makes it easier to determine whether standard RL techniques are applicable to a given problem, or whether more specialised methods must be used. In particular, our results show that there are situations in which careful reward specification and reward shaping will not be sufficient to robustly incentivise the desired behaviour. In those cases, we must instead use an alternative policy synthesis method, such as e.g.\ those offered by MORL. Secondly, in the area of reward learning, most algorithms attempt to fit a scalar, Markovian reward function to their training data \cite[e.g.\ ][]{christiano2023deep}. Our work clarifies the implicit modelling assumptions behind these algorithms, and shows that there are many situations in which these models will be misspecified.

Our work also suggests several directions for further work. The fact that the common settings of MORL and risk-sensitive RL indeed are genuine extensions over the standard (scalar, Markovian) setting provides additional motivation for further work in these areas. Our work also suggests that it could be interesting to further explore the modal setting, or other directions that aim to extend the expressivity of the standard RL setting. We give an overview of the existing work in this area in Section~\ref{section:related_work}. Our work also motivates work on reward learning algorithms which do not assume that the preferences of the demonstrator can be captured by a scalar, Markovian reward. There is some existing work in this area \cite[e.g.\ ][]{abate2022learning}, but it remains quite limited. Moreover, another interesting direction for further work would be to quantify the consequences of taking a task which cannot be perfectly represented using a Markovian reward function, and trying to approximate it using a Markovian reward function. For example, could we bound the worst-case regret that might be incurred if a MORL problem is approximated using a scalar reward? Finally, another interesting direction for further work would be to more thoroughly explore the expressivity of other types of problem settings, and their relationship to each other.


\subsection{Related Work}\label{section:related_work}

There has been a lot of recent work on the expressivity of Markovian reward functions.
Here, we summarise relevant contributions, and detail differences with our work.  

Notably, there are three recent papers which provide necessary and sufficient conditions for when a particular type of task can be expressed using a particular type of reward function. The first of these is \cite{Pitis_2019}, who consider a task to be a preference relation defined over \emph{prospects}, where a prospect is defined as a pair of a state and a policy. Moreover, they generalise the discount function by allowing it to depend on the transition (instead of always being a constant value $\gamma$). They then add two axioms (and one assumption) to the famous vNM-axioms (from \cite{vonneumann1947}), to obtain necessary and sufficient conditions for when a task (as they formalise it) can be expressed as a Markovian reward with transition-dependent discounting.
Our work differs from their in several ways, as explained shortly.

The next paper is \cite{pmlr-v162-shakerinava22a}, who provide an alternative, simpler axiomatisation of the setting considered by \cite{Pitis_2019}, and also provide further axioms to describe two additional types of environments. They consider environments without any discount factor, but instead use termination probabilities, which can be used to simulate the standard case with exponential discounting.

The third paper is \cite{settling_reward_hypothesis}, who generalise the results of \cite{pmlr-v162-shakerinava22a} even further, and provide an alternative axiom to add to the vNM axioms.
They start by considering preference relations over finite trajectories, and then extend this to a preference relation over policies by saying that a policy $\pi_1$ is preferred to $\pi_2$ if there exists a time $t$ after which the trajectory distribution induced by $\pi_1$ is always preferable to the trajectory distribution induced by $\pi_2$. This encompasses the setting with exponentially discounted reward, the setting with limit-average reward, and the episodic setting. They consider both the case where the discount function is transition-dependent, and the case when it is constant.

Our work differs from that by \cite{Pitis_2019, pmlr-v162-shakerinava22a, settling_reward_hypothesis} in a few ways.
First of all, these papers aim to establish general necessary and sufficient conditions for when a task can be formalised as a Markovian reward, whereas we instead focus on three specific classes of tasks that we believe to be especially interesting. 
It might in principle be possible to derive our results as a special case of theirs. However, doing this would be quite non-trivial, and possibly more difficult than our direct derivations. 
Secondly, the axiomatisations provided by \cite{Pitis_2019, pmlr-v162-shakerinava22a, settling_reward_hypothesis} are difficult to use in practice. Our results, on the other hand, are arguably intuitive to understand, and concern some settings that are both popular and important.
Our work could thus be construed as a study on the practical consequences of the work by \cite{Pitis_2019, pmlr-v162-shakerinava22a, settling_reward_hypothesis}, with results that may be more directly useful to practitioners.
There are also several differences in how we formalise the problem compared to \cite{Pitis_2019, pmlr-v162-shakerinava22a, settling_reward_hypothesis}. For example, we consider the case with fixed discount rates, whereas \cite{Pitis_2019} and \cite{settling_reward_hypothesis} consider transition-dependent discount rates. To give another example, \cite{pmlr-v162-shakerinava22a} consider finite trajectories, whereas we consider infinite trajectories (noting that the latter can model the former, but not vice versa).
These differences further contribute to distinguishing our results from theirs.

Another notable piece of related work is \citet{markovrewardexpressivity}, who point to three different ways to formalise the notion of a \enquote{task} (namely, as a set of acceptable policies, as an ordering over policies, or as an ordering over trajectories). They then demonstrate that each of these classes contains at least one instance which cannot be expressed using a Markovian reward function, and provide algorithms which compute reward functions for these types of tasks.
Our work is different from theirs in a few different ways. 
First of all, we consider three different ways to specify a policy ordering, and then derive necessary and sufficient conditions which can be used to directly determine when the resulting policy ordering can be expressed as a Markovian reward function. \citet{markovrewardexpressivity} do not provide necessary and sufficient conditions, but instead only provide a counter-example for each type of task, showing that Markovian rewards cannot formalise all tasks of that type.

Another important paper is the work by  \cite{Vamplew2022}, who argue that there are many important aspects of intelligence which can be captured by MORL, but not by scalar RL. 
Like them, we also argue that MORL is a genuine extension of scalar RL, but our approach is quite different. They focus on the question of whether MORL or (scalar) RL is a better foundation for the development of general intelligence (considering feasibility, safety, and etc), and they provide qualitative arguments and biological evidence. By contrast, we are more narrowly focused on what incentive structures can be expressed by MORL and scalar RL, and our results are mathematical.

\cite{miura2022} considers the question of when a task can be expressed as a constrained MDP (CMDP), or as a Markovian reward. They formalise a task as two sets of policies, $\langle \Pi_G, \Pi_B \rangle$, and consider a CMDP to express the task if all policies in $\Pi_G$, and none of the policies in $\Pi_B$, are feasible, and consider a Markovian reward to express the task if all policies in $\Pi_G$, and none of the policies in $\Pi_B$, are optimal under that reward. They then derive necessary and sufficient conditions for both of these cases, and show that CMDPs are strictly more expressive than Markovian rewards for these types of tasks. The CMDP framework is a special case of the MORL framework we discuss in Section~\ref{section:morl}, roughly corresponding to the \textbf{MaxSat} objective. On the other hand, we formalise the notion of a task as a policy ordering, whereas \cite{miura2022} formalises it as a set of feasible policies.

Also relevant is the work by \cite{pitis2022rational}, who consider a task to consist of multiple Markovian reward functions, each of which may use a different discount parameter, and where the goal is to maximise the sum of these rewards. They then show that this setting may lead to the optimal policy being non-stationary, which demonstrates that it cannot always be expressed using Markovian rewards. Our analysis of the MORL setting allows for more general objectives than the case where the goal is to maximise the sum of the individual rewards. On the other hand, we assume that the same discount parameter is used for each reward. Our analysis is therefore in some ways more general, and in other ways more restrictive, than that of \cite{pitis2022rational}.

Also related is the work by \cite{rewardgaming}, who demonstrate that if for two rewards $R_1$, $R_2$ there are no policies $\pi_1$, $\pi_2$ such that $J_1(\pi_1) < J_2(\pi_2)$ and $J_2(\pi_1) > J_2(\pi_2)$, then either $R_1$ and $R_2$ are equivalent, or one of them is trivial. This means that there are some policy orderings that cannot be expressed using Markovian rewards. 
We consider different kinds of policy orderings than they do.

There is also other relevant work that is less strongly related. For example, \cite{reward_machines} point out that there are certain tasks which cannot be expressed using Markovian rewards, and propose a way extend their expressivity by augmenting the reward function with an automaton that they call a \emph{reward machine}. 
Similar approaches have also been used by \cite{bcHKA20,hagw21}, tackling infinite-horizon tasks for single- and multi-agent systems. 
There are also other ways to extend Markovian rewards to a more general setting, such as \emph{convex RL}, as studied by e.g.\ \cite{convex_1,convex_2,convex_3,convex_4, convex_5}, and \emph{vectorial RL}, as studied by e.g.\ \cite{vectorial_1,vectorial_2}.
Analysing the expressivity of these problem settings more extensively would be an interesting direction for further work.

There is a large literature on (the overlapping topics of) single-policy MORL, constrained RL, and risk-sensitive RL. 
These areas are too large for it to be possible to give a fully complete overview of this work here.
Some notable examples include \citet{cpo, chow2015riskconstrained, miryoosefi2019reinforcement, tessler2019, lmorl}. 
This existing literature typically focuses on the creation of algorithms for solving particular MORL problems, rather than on characterising when MORL problems can (or cannot) be reduced to scalar RL.
Modal RL has (to the best of our knowledge) never been discussed explicitly in the literature before. However, it relates to some existing work, such as side-effect avoidance
\citep{krakovnaside, krakovnaside2, turnerside, griffin2022alls}, and the work by \cite{pctl_rl}.

Finally, our work also relates to existing work in decision theory, social choice theory, and related fields. This of course includes the famous work by \cite{vonneumann1947}. As discussed previously, the work by \cite{harsanyi1955cardinal} is also particularly relevant. Note that work in decision theory and social choice theory typically only considers single-step decision problems, whereas the RL setting of course considers sequential decision making. There are also a few other modelling assumptions that are common in decision theory and social choice theory which do not hold in the RL setting.
For example, in these fields, it is common to assume that the choice set is finite (whereas the set of trajectories in RL may be infinite), that preferences are defined over all distributions over the choice set (whereas it in RL is more common to only consider distributions that can be realised by some policy for a given transition function), and that a utility function can be any function from the choice set to real numbers (whereas many of these functions cannot be expressed as reward functions).
Consequently, results from decision theory and social choice theory only sometimes generalise to the RL setting. For example, in Section~\ref{section:morl}, we provide some examples of results that do generalise to the RL setting, and in Section~\ref{section:risk_sensitive_rl}, we provide some examples of results which do not generalise.

\bibliography{references}

\newpage
\appendix

\include{appendix}

\end{document}

%% file: appendix.tex
\setcounter{theorem}{0}
\setcounter{corollary}{0}
\appendix

\section{Proofs}\label{appendix:proofs}

Here, we will provide all proofs that were omitted from the main text. We will begin with Theorem~\ref{thm:linearity_thm}, from Section~\ref{section:morl}.

\begin{theorem}
If a MOMDP $\M$ with objective $\Ob$ is scalarizable, then there exist $w_1 \dots w_k \in \mathbb{R}$ such that $\M$ with $\Ob$ is scalarized by the reward $R(s,a) = \sum_{i=1}^k w_i \cdot R_i(s,a)$.
\end{theorem} 

To prove this, we must first set up some theoretical preliminaries. For convenience, let $n = |S||A|$, let $T = \SxA$, and let each transition in $\SxA$ be indexed by an integer $i \in [1,n]$. Moreover, given a reward function $R$, let $\Vec{R} \in \mathbb{R}^n$ be the vector such that $\Vec{R}_i = R(T_i)$. Next, given $\tau$, $\init$, and $\gamma$, let $\m : \Pi \to \mathbb{R}^n$ be the function where 
$$
\m(\pi)_i = \sum_{t=0}^{\infty} \y^t \mathbb{P}_{\xi \sim \pi}(\xi = T_i).
$$
Now $J(\pi) = \vec{R} \cdot \m(\pi)$. In other words, this construction lets us decompose $J$ into two steps, the first of which embeds $\pi$ in $\mathbb{R}^n$, and the second of which is a linear function. 
Let $S_\gamma$ be the smallest affine subspace of $R^n$ such that $\mathrm{Im}(\m) \in S_\gamma$.
We will also use the following lemma:

\begin{lemma}\label{lemma:open_set}
$\mathrm{Im}(\m)$ is open in $S_\gamma$.
\end{lemma}

For a proof of Lemma~\ref{lemma:open_set}, see \cite{IRLmisspecification} (their Lemma A.11).
We can now prove Theorem~\ref{thm:linearity_thm}:

\begin{proof}
Suppose the MOMDP $\MOMDP$ with $\Ob$ is equivalent to the MDP $\MDP$.

First, note that $J(\pi) = \vec{R} \circ \m (\pi)$, and that $J_i (\pi) = \vec{R_i} \circ \m (\pi)$ for each of $R_i \in \Rs$. 
Let $M$ be the $(n \times k)$-dimensional matrix that maps each vector $x \in \mathbb{R}^n$ to $\langle R_1 \cdot x, \dots, R_k \cdot x\rangle$. In other words, $M$ is the matrix whose rows are $\vec{R}_1 \dots \vec{R}_k$.
Since $J(\pi)$ is a function of $J_1(\pi) \dots J_k(\pi)$, we have that $\vec{R} \cdot x_1 = \vec{R} \cdot x_2$ if $M \cdot x_1 = M \cdot x_2$ for any $x_1, x_2 \in \mathrm{Im}(\m)$.

We will first show that $\vec{R} \cdot x_1 = \vec{R} \cdot x_2$ if $M \cdot x_1 = M \cdot x_2$ for any $x_1, x_2 \in S_\gamma$, not just any $x_1,x_2 \in \mathrm{Im}(\m)$.
Let $x_1, x_2$ be any two points in $S_\gamma$ such that $M \cdot x_1 = M \cdot x_2$, and let $x$ be some arbitrary element of $\mathrm{Im}(\m)$. Let $y_1 = x_1 - x$ and $y_2 = x_2 - x$. Since $\mathrm{Im}(\m)$ is open in $S_\gamma$ (as per Lemma~
\ref{lemma:open_set}), there is an $\alpha > 0$ such that $x + \alpha \cdot y_1 \in \mathrm{Im}(\m)$ and $x + \alpha \cdot y_2 \in \mathrm{Im}(\m)$. Since $M$ is linear, and since $M \cdot x_1 = M \cdot x_2$, we have that $M \cdot (x + \alpha \cdot y_1) = M \cdot (x + \alpha \cdot y_2)$. Moreover, since $x + \alpha \cdot y_1 \in \mathrm{Im}(\m)$ and $x + \alpha \cdot y_2 \in \mathrm{Im}(\m)$, this means that $\vec{R} \cdot (x + \alpha \cdot y_1) = \vec{R} \cdot (x + \alpha \cdot y_2)$. Finally, from the properties of linear functions, this in turn implies that $\vec{R} \cdot x_1 = \vec{R} \cdot x_2$. Thus, if $M \cdot x_1 = M \cdot x_2$ then $\vec{R} \cdot x_1 = \vec{R} \cdot x_2$ for all $x_1, x_2 \in S_\gamma$.

Next, note that we can decompose $M$ into two matrices $M_1, M_2$ such that $M = M_1 \cdot M_2$, where $M_1$ is invertible, and $M_2$ is an orthogonal projection such that $M_2(x_1) = M_2(x_2)$ if and only if $M(x_1) = M(x_2)$. This means that $\vec{R} \cdot x = \vec{R} \cdot M_2(x)$ for all $x \in S_\gamma$. From this, we obtain that $\vec{R} \cdot x = \vec{R} \cdot M_1^{-1} \cdot M_1 \cdot M_2(x) = \vec{R} \cdot M_1^{-1} \cdot M(x)$ for all $x \in S_\gamma$. Since $\vec{R} \cdot M_1^{-1}$ is a linear function, this means that $\vec{R} \cdot x$ can be expressed as $\sum_{i=1}^k w_i \cdot M(x)_i$ for some $w_1 \dots w_k$ for all $x \in S_\gamma$. 

Recall that $J(\pi) = \vec{R} \cdot \m(\pi)$, where $m(\pi) \in S_\gamma$.
This means that $J(\pi) = \sum_{i=1}^k w_i \cdot M(\m(\pi))_i = \sum_{i=1}^k w_i \cdot \vec{R_i} \cdot \m(\pi) = \sum_{i=1}^k w_i \cdot J_i(\pi)$. This completes the proof.
\end{proof}

\begin{corollary}
If $\Ob(J_1 \dots J_k)$ has a non-linear representation $U$, and $\M$ is a MOMDP whose $J$-functions are $J_1 \dots J_k$, then $\M$ with $\Ob$ is not equivalent to any MDP.
\end{corollary}
\begin{proof}
Assume for contradiction that $\M$ with $\Ob$ is equivalent the MDP $\MDP$. Then $J$ represents $\Ob(J_1 \dots J_k)$, and this in turn means that $U$ must be strictly monotonic in $J$. Moreover, Theorem~\ref{thm:linearity_thm} implies that $J = \sum_{i=0}^k w_i \cdot J_i$ for some $w_1 \dots w_k \in \mathbb{R}^k$. However, this contradicts our assumptions.
\end{proof}

\begin{corollary}
There is no MDP equivalent to $\M$ with $\textbf{LexMax}$, as long as $\M$ has at least two reward functions that are neither trivial, equivalent, or opposite. 
\end{corollary}
\begin{proof}

Suppose $\M$ with $\texttt{LexMax}$ is equivalent to $\tilde{\M} = \MDPwO$. Let $i$ be the smallest number such that $R_i$ is non-trivial, and let $j$ be the smallest number greater than $i$ such that $R_j$ is non-trivial, and not equivalent to or opposite of $R_i$. Then there are $\pi_1,\pi_2$ such that $J_i(\pi_1) = J_i(\pi_2)$ and $J_j(\pi_1) < J_j(\pi_2)$, which means that $\pi_1 \prec_\texttt{Lex}^\M \pi_2$.
Moreover, since $\tilde{J}$ represents $\prec_\texttt{Lex}^\M$, it follows that there are no $\pi, \pi'$ such that $J_i(\pi) < J_i(\pi')$ and $\tilde{J}(\pi) > \tilde{J}(\pi')$. Then Theorem 1 in \citet{rewardgaming} implies that $R_i$ is equivalent to $\tilde{R}$. However, then $\tilde{J}(\pi_1) = \tilde{J}(\pi_2)$, which means that $\tilde{J}$ cannot represent $\prec_\texttt{Lex}^\M$.
\end{proof}


\begin{corollary}
There is no MDP equivalent to $\M$ with $\textbf{MaxMin}$, unless $\M$ has a reward function $R_i$ such that $J_i(\pi) \leq J_j(\pi)$ for all $j \in \{1 \dots k\}$ and all $\pi$.
\end{corollary}
\begin{proof}
$\Ob_\texttt{Min}^\M$ is represented by 
the function 
$U(\pi) = \mathrm{min}_i J_i(\pi)$. Moreover, if $\M$ has no reward function $R_i$ such that $J_i(\pi) \leq J_j(\pi)$ for all $j \in \{1 \dots k\}$ and all $\pi$ then this representation is non-linear. 
Corollary~\ref{cor:nonlinear_rep} then implies that $\M$ with $\texttt{MaxMin}$ is not equivalent to any MDP. 
\end{proof}

\begin{corollary}
There is no MDP equivalent to $\M$ with $\textbf{MaxSat}$, as long as $\M$ has at least one reward $R_i$ where $J_i(\pi_1) < c_i$ and $J_i(\pi_2) \geq c_i$ for some $\pi_1, \pi_2 \in \Pi$.
\end{corollary}
\begin{proof}
Note that $\texttt{MaxSat}(\M)$ is represented by the function $U(\pi) = \sum_{i=1}^k \mathbbm{1}[J_i(\pi) \geq c_i]$, where $\mathbbm{1}[J_i(\pi) \geq c_i]$ is the function that is equal to $1$ when $J_i(\pi) \geq c_i$, and $0$ otherwise. 
Moreover,
$U$ is not strictly monotonic in any function that is linear in $J_1 \dots J_k$. Corollary \ref{cor:nonlinear_rep} thus implies that $\M$ with $\texttt{MaxSat}$ is not equivalent to any MDP.
\end{proof}

\begin{corollary}
There is no MDP equivalent to $\M$ with $\textbf{ConSat}$, unless either $R_1$ and $R_2$ are equivalent, or $\max_{\pi}J_1(\pi) \leq c$.
\end{corollary}
\begin{proof}
$\Ob_\texttt{Con}^\M$ is represented by
$U(\pi) = \{J_1(\pi) \text{ if } J_1(\pi) \leq c \text{, else } J_2(\pi) - \min_\pi J_2(\pi) + c\}$. Moreover, this representation is non-linear, unless either $R_1$ and $R_2$ are equivalent, or $\max_{\pi}J_1(\pi) \leq c$.
Corollary~\ref{cor:nonlinear_rep} then implies that $\M$ with $\texttt{ConSat}$ is not equivalent to any MDP. 
\end{proof}

We next give the proof of Theorem~\ref{thm:risk_theorem}, from Section~\ref{section:risk_sensitive_rl}.

\begin{theorem}
Given $\States$, $\Actions$, and $\gamma$, 
let $R_1$ and $R_2$ be two reward functions.
If for all $\xi_1,\xi_2 \in (\SxA)^\omega$ and $\gamma \geq 0.5$, 
$$
G_1(\xi_1) \leq G_1(\xi_2) \iff G_2(\xi_1) \leq G_2(\xi_2),
$$
then there exist $a \in \mathbb{R}$, $b \in \mathbb{R} > 0$ such that for all $\xi \in (\SxA)^\omega$,
$$
G_1(\xi) = b \cdot G_2(\xi) + a.
$$
\end{theorem}

\begin{proof}
We can first note that if $G_1$ is constant then $G_2$ must also be constant, and vice versa, in which case this result is straightforward (with $b = 1$, $a = G_1 - G_2$). For the rest of the proof, assume that neither $G_1$ or $G_2$ is constant.

For convenience, let $n = |S||A|$, let $T = \SxA$ , and let each transition in $\SxA$ be indexed by an integer $i \in [1,n]$. Let $\Vec{R_1} \in \mathbb{R}^n$ be the vector such that $\Vec{R_1}_i = R_1(T_i)$, and $\Vec{R_2} \in \mathbb{R}^n$ be the vector such that $\Vec{R_2}_i = R_2(T_i)$. Moreover, let $m : T \to \mathbb{R}^n$ be the function where 
$$
m(\xi)_i = \sum_{j=0}^\infty \delta^j \mathbbm{1}[\xi_j = T_i].
$$
Now $G_1(\xi) = \vec{R_1} \cdot m(\xi)$ and $G_2(\xi) = \vec{R_2} \cdot m(\xi)$. In other words, this construction lets us decompose $G_1$ and $G_2$ into two steps, the first of which embeds $\xi$ in $\mathbb{R}^n$, and the second of which is a linear function.

Next, let us consider what $\mathrm{Im}(m)$ looks like. First, note that $m(\xi)_i \geq 0$ for all $i$ and all $\xi$. Next, note that $\sum m(\xi) = 1/(1-\gamma)$ for all $\xi$. This means that $\mathrm{Im}(m)$ is located inside the simplex that is formed by all points in the positive quadrant of $\mathbb{R}^n$ whose $L_1$-norm is $1/(1-\gamma)$.

Consider two arbitrary transitions $t_i, t_j \in T$. Note that $m(t_i^\omega)$ is the point where the aforementioned simplex intersects the $i$'th basis vector of $\mathbb{R}^n$, and similarly for $m(t_j^\omega)$. Moreover, if $\xi$ is made up entirely from $t_i$ and $t_j$ in some combination and order (i.e., $\xi \in \{t_i,t_j\}^\omega \subseteq T$), then $m(\xi)$ is on the line between $m(t_i^\omega)$ and $m(t_j^\omega)$. 

Let $\alpha$ be any number in $[0, 1/(1-\gamma)]$. Since $1/\gamma > 1$, there is a representation of $\alpha$ in base $1/\gamma$.
This means that there is an integer $u$ and a sequence of integers $\{a_k\}_{k \in (-\infty, u]}$ such that
$$
\sum_{k = u}^{-\infty} a_k \cdot (1/\gamma)^k = \alpha
$$
where each $a_k$ is a nonnegative integer less than $1/\gamma$. Since $\gamma \geq 0.5$, this means that each $a_k$ is 0 or 1. Moreover, since $\alpha \leq 1/(1-\gamma)$, we have that $u \leq 0$. By rewriting using $k' = -k$, this means that there is a sequence $\{a_{k'}\}_{k' \in [0,\infty)}$ where each $a_{k'} \in \{0,1\}$ such that
$$
\sum_{k' = 0}^{\infty} a_{k'} \cdot \gamma^{k'} = \alpha.
$$
Let $\xi \in T$ be the trajectory where $\xi_{k'} = t_i$ if $a_{k'} = 1$, and $t_j$ if $a_{k'} = 0$. We now have that $m(\xi) = \alpha/(1/(1-\gamma)) \cdot m(t_i^\omega) + (1-\alpha/(1/(1-\gamma))) \cdot m(t_j^\omega)$. Since $\alpha$ was chosen arbitrarily from $[0, 1/(1-\gamma)]$, this means that every point on the line between $m(t_i^\omega)$ and $m(t_j^\omega)$ are in $\mathrm{Im}(m)$. Since $t_i$ and $t_j$ were also chosen arbitrarily, this holds for any $t_i$ and $t_j$ in $T$.

Consider again the simplex that is formed by all points in the positive quadrant of $\mathbb{R}^n$ whose $L_1$-norm is $1/(1-\gamma)$. We have just shown that every point on the edges (1-faces) of this simplex are in $\mathrm{Im}(m)$.

Consider the linear functions that $\vec{R_1}$ and $\vec{R_2}$ induce on $\mathbb{R}^n$. Take the point $x$ at the centre of the simplex, and consider the tangent plane of $\vec{R_1}$ at this point. Since every point on any of the simplex edges are in $\mathrm{Im}(m)$, we have that this tangent plane must intersect $\mathrm{Im}(m)$ at $n-1$ linearly independent points. Since $\vec{R_1}\cdot x_1 = \vec{R_1}\cdot x_2$ implies that $\vec{R_2}\cdot x_1 = \vec{R_2}\cdot x_2$ for all $x_1,x_2 \in \mathrm{Im}(m)$, we have that the tangent plane of $\vec{R_2}$ at $x$ must intersect $\mathrm{Im}(m)$ at the same points. This implies that there are $a, b \in \mathbb{R}$ such that $G_1 = b \cdot G_2 + a$. Since moreover $G_1(\xi_1) \leq G_1(\xi_2) \iff G_2(\xi_1) \leq G_2(\xi_2)$, we have that $b > 0$.
\end{proof}

\begin{theorem}
For any modal reward $R^\Diamond$ and any transition function $\tau$, there exists a reward $R$ that is contingently equivalent to $R^\Diamond$ given $\tau$. Moreover, unless $R^\Diamond$ is trivial, there is no reward that is robustly equivalent to $R^\Diamond$.
\end{theorem}
\begin{proof}
This is straightforward.
For the first part, simply let $R(s,a,s') = R^\Diamond(s,a,s',\tau)$. 
The second part is immediate from the definition of trivial modal reward functions.
\end{proof}

\section{Tasks as Optimal Policies}\label{appendix:towards_nas}

In this paper, we primarily think of a \enquote{task} as corresponding to a policy ordering.
An alternative way to formalise the notion of a task is as a set of optimal policies. It is fairly straightforward to provide necessary and sufficient conditions for when this type of task can be expressed using a scalar, Markovian reward function.

\begin{proposition}
A set of policies $\hat{\Pi}$ is the optimal policy set for some reward if and only if there is a function $o : \mcS \to \mathcal{P}(\mcA) \setminus \varnothing$ that maps each state to a (non-empty) set of \enquote{optimal actions}, and $\pi \in \hat{\Pi}$ if and only if $\mathrm{supp}(\pi(s)) \subseteq o(s)$.
\end{proposition}

\begin{proof}
For the \enquote{if} part, consider the reward function $R$ where $R(s,a,s') = 0$ if $a \in o(s)$, and $R(s,a,s') = -1$ otherwise. 
The \enquote{only if} part follows from the fact that the optimal $Q$-function $Q^\star$ is the same for all optimal policies, 
so we can let $o(s) = \mathrm{argmax}_a Q^\star(s,a)$.
\end{proof}

We can see that some tasks of this form cannot be expressed by Markovian rewards. For example, consider the task \enquote{always go in the same direction} --- this task cannot be expressed as a reward function, because any policy that mixes the actions of two other optimal policies must itself be optimal. It also shows that Markovian reward functions cannot be used to encourage \emph{stochastic} policies. For example, there is no Markovian reward function under which \enquote{play rock, paper, and scissors with equal probability} is the unique optimal policy.

\section{More MORL Objectives}\label{appendix:morl_objectives}

In this Appendix, we give even more examples of MORL objectives, and some comments on how to construct them -- the purpose of this is mainly just to show how rich this space is. First, similar to the MaxMin objective, we might want to judge a policy according to its \emph{best} performance:

\begin{definition}\label{def:maxmax}
Given $J_1 \dots J_k$, the \textbf{MaxMax} objective $\prec_\texttt{Max}$ is given by $\pi_1 \prec_\texttt{Max} \pi_2 \iff \max_i J_i(\pi_1) < \max_i J_i(\pi_2)$.
\end{definition}


We would next like to point out that it is possible to create smooth versions of almost any MORL objective. In Section~\ref{section:solving_inexpressible}, we outline an approach for learning any continuous, differentiable MORL objective, so this is quite useful. We begin with a soft version of the MaxMax objective:

\begin{definition}\label{def:maxsoft}
Given $J_1 \dots J_k$ and $\alpha > 0$, the \textbf{Soft MaxMax} objective $\prec_\texttt{MaxSoft}$ is given by 
$$
J_\texttt{MaxSoft}(\pi) = \left(\sum_{i=1}^k J_i(\pi) e^{\alpha J_i(\pi)}\middle)\middle/\middle(\sum_{i=1}^k e^{\alpha J_i(\pi)}\right).
$$
\end{definition}

This is of course not the only way to continuously approximate MaxMax, it is just an example of one way of doing it. Here $\alpha$ controls how \enquote{sharp} the approximation is -- the larger $\alpha$ is, the closer $J_\texttt{MaxSoft}$ gets to the sharp max function, and the smaller $\alpha$ is, the closer it gets to the arithmetic mean function (so by varying $\alpha$, we can continuously interpolate between them). Similarly, we can also create a smooth version of MaxMin:

\begin{definition}\label{def:minsoft}
Given $J_1 \dots J_k$ and $\alpha > 0$, the \textbf{Soft MaxMin} objective $\prec_\texttt{MinSoft}$ is given by 
$$
J_\texttt{MinSoft}(\pi) = \left(\sum_{i=1}^k J_i(\pi) e^{-\alpha J_i(\pi)}\middle)\middle/\middle(\sum_{i=1}^k e^{-\alpha J_i(\pi)}\right).
$$
\end{definition}

As before, the larger $\alpha$ is, the closer $J_\texttt{MinSoft}$ gets to the sharp min function, and the smaller $\alpha$ is, the closer it gets to the arithmetic mean function
We can also smoothen MaxSat:

\begin{definition}\label{def:satsoft}
Given $J_1 \dots J_k$, $c_1 \dots c_k$, and $\alpha > 0$, the \textbf{Soft MaxSat} objective $\prec_\texttt{SatSoft}$ is 
$$
J_\texttt{SatSoft}(\pi) = \sum_{i=1}^k \left(\frac{1}{1 + e^{-\alpha(J_i(\pi) - c_i)}}\right).
$$
\end{definition}

The larger $\alpha$ is, the closer $J_\texttt{SatSoft}$ gets to the sharp MaxSat function (and the smaller $\alpha$ gets, the closer $J_\texttt{SatSoft}$ gets to a flat $0.5$). And, again, this is of course not the only way to create a smooth version of MaxSat.
It is unclear if it is possible to create a smooth version of ConSat without having any prior knowledge of (a lower bound of) the value of $\min_\pi J_1(\pi)$, but with this value it should be reasonably straightforward (see the construction in Corollary~\ref{corollary:no_consat}). As for LexMax, we can of course create a smooth approximation of it by taking a linear approximation of the weights, but here we would need some prior knowledge of $\max_\pi J_1(\pi) \dots \max_\pi J_k(\pi)$.

\section{A Method for Solving Modal Tasks}\label{appendix:solve_modal}

In this Appendix, we give an outline of one possible method for solving modal tasks.
We mainly want to show that it is \emph{feasible} to learn modal tasks, and so we only provide a solution sketch; the task of \emph{implementing} and \emph{evaluating} this method is something we leave as a topic for future work.

We will first define a restricted class of modal tasks, which is both very expressive, and also more amenable to learning than the more general version given in Definition~\ref{definition:modal_reward}:

\begin{definition}\label{def:affordance_mdp}
An \emph{affordance} consists of a reward function and a discount factor, $\langle R,\gamma \rangle$, and an \emph{affordance-based reward} is a function $R^\Diamond : \SxAxS \times \mathbb{R}^{2k} \to \mathbb{R}$, that is continuous in the last $2k$ arguments. 
An \emph{affordance-based MDP} is a tuple $\langle \mcS, \mcA, \tau, \mu_0, R^\Diamond, \gamma, \langle R,\gamma \rangle^k \rangle$, where the reward given for transitioning from $s$ to $s'$ via $a$ is $R^\Diamond(s,a,s',V_1^\star(s) \dots V_k^\star(s), V_1^\star(s') \dots V_k^\star(s'))$, where $V_i^\star$ is the optimal value function of the $i$'th affordance.
\end{definition}

This definition requires some explanation. 
In psychology (and other fields, such as user interface design), an affordance is, roughly, a perceived possible action, or a perceived way to use an object.
For example, if you see a button, then the fact that you can \emph{press} that button, and expect something to happen, is part of \emph{how you perceive} it, in a way that might not be the case if you could somehow show the button to a premodern human.
It can also be used to refer to a choice or action that is perceived as available in some context (without being tied to an object). Here, we are using it to refer to a \emph{task} that could be performed in an MDP. The intuition is that $R^\Diamond$ is allowed to depend on what \emph{could be done} from $s$ and $s'$, in addition to the state features of $s$ and $s'$.

Before outlining an algorithm, let us first give a few examples of how to formalise modal tasks within this framework. First consider the instruction \enquote{you should always be able to return to the start state}. We can formalise this using a reward function $R_1$ that gives $1$ reward if the start state is entered, and $0$ otherwise, and pair it up with a discount parameter $\gamma$ that is very close to $1$. We could then set $R^\Diamond$ to, for example, $R^\Diamond(s,a,s',V_1^\star(s), V_1^\star(s')) = R(s,a,s') \cdot \tanh(V_1^\star(s'))$, where $R$ describes some base task. In this way, no reward is given if the start state cannot be reached from $s'$. Next, consider the instruction \enquote{never enter a state from which it is possible to quickly enter an unsafe state}.
To formalise this, let $R_1$ give $1$ reward if an unsafe state is entered, and $0$ otherwise, and let $\gamma$ correspond to a very high discount rate (e.g.\ $0.7$). 
We could then set $R^\Diamond$ to, for example, $R^\Diamond(s,a,s',V_1^\star(s), V_1^\star(s')) = R(s,a,s') - V_1^\star(s')$, where $R$ again describes some base task.

These examples show that our \enquote{affordance-based} MDPs are quite flexible, and that they should be able to formalise many natural modal tasks in a satisfactory way, including most of our motivating examples.\footnote{This arguably excludes \enquote{you should never enter a state where you would be unable to receive a feedback signal}. However, this instruction only makes sense in a multi-agent setting.} However, the definition could of course be made more general. For example, we could allow the affordances to themselves be based on affordance-based reward functions, etc. However, it is not clear if this would bring much benefit in practice. 

Let us now outline an approach for solving affordance-based MDPs using reinforcement learning, specifically using an action-value method. 
First, let the agent maintain $k+1$ $Q$-functions, $Q^\Diamond, Q_1, \dots, Q_k$, one for $R^\Diamond$ and one for each affordance $\langle R_i, \gamma_i \rangle$. Next, we suppose that the agent updates each of $Q_1, \dots, Q_k$ using an off-policy update rule, such as $Q$-learning; this will ensure that $Q_1, \dots, Q_k$ converge to their true values (i.e.\ to $Q_1^\star \dots Q_k^\star$), as long as the agent explores infinitely often. Note that the use of an off-policy update rule is crucial. 
Next, let the agent update $Q^\Diamond$ as if it were an ordinary Markovian reward function, using the reward $\hat{R}(s,a,s') = R^\Diamond(s,a,s',V_1(s) \dots V_k(s), V_1(s') \dots V_k(s'))$, where $V_i(s)$ is given by $\max_a Q_i(s,a)$.
In other words, we let it update $Q^\Diamond$ using an \emph{estimate} of the true value of $R^\Diamond$, expressed in terms of its current estimates of $V_1^\star \dots V_k^\star$. 
The fact that $Q_1, \dots, Q_k$ converge to $Q_1^\star, \dots, Q_k^\star$, and the fact that $R^\Diamond$ is continuous in its value function arguments, will ensure that the estimate $\hat{R}$ also converges to the true value of $R^\Diamond$.
The update rule used for $Q^\Diamond$ could be either on-policy or off-policy.
We then suppose that the agent selects its actions by applying a Bandit algorithm to $Q^\Diamond$, and that this Bandit algorithm is greedy in the limit, but also explores infinitely often, as usual.

This algorithm should be able to learn to optimise the reward in any affordance-based MDP. 
In the tabular case, it should be possible (and reasonably straightforward) to prove that it always converges to an optimal policy (assuming that appropriate learning rates are used, etc), using Lemma 1 in \cite{Singh2000}. We would also expect it to perform well in practice, when used with function approximators (such as neural networks). However, we leave the task of implementing and properly evaluating this approach as a topic for future work.

There are also several ways that this algorithm could be tweaked or improved. For example, the algorithm we have described is an action-value algorithm, but the same approach could of course be used to make an actor-critic algorithm instead.
We also suspect that there could be interesting modifications one could make to the exploration strategy of the algorithm.
If a standard Bandit algorithm (such as $\epsilon$-greedy) is used, then the agent will mostly take actions that are optimal under its current estimate of $Q^\Diamond$. In the ordinary case, this is good, because it leads the agent to spend more time in the parts of the MDP that are relevant for maximising the reward. However, in this case, there is a worry that it could lead the agent to neglect the parts of the (affordance-based) MDP that are relevant for learning more about $V_1^\star \dots V_k^\star$, which might slow down the learning.
Again, we leave such developments for future work, since our aim here only is to show that it is feasible to learn non-trivial modal tasks.

We also want to point out that the work by \citet{pctl_rl} could provide another starting point for learning modal tasks using RL. In their work, they present some RL-based methods for determining whether a specification in Probabilistic Computational Tree Logic (PCTL) holds in an MDP. PCTL can be used to specify many kinds of properties of states in MDPs which depend on the transition function, including e.g.\ what states can and cannot be reached from a particular state, and with what probability, etc. We can therefore specify non-trivial modal tasks by providing a number of PCTL formulas, and allowing the reward function to depend on the truth values of these formulas. That is, we could consider a setup that is analogous to that which we give in Definition~\ref{def:affordance_mdp}, but where the \enquote{affordances} are replaced by PCTL formulas. It should then be possible to learn tasks specified in this manner by using the techniques of \citet{pctl_rl} to learn the values of the PCTL formulas, and then using ordinary RL to train on the resulting reward function.